\documentclass[12pt]{article}

\usepackage{fullpage}

\usepackage{natbib}

\usepackage{algorithm}
\usepackage{algorithmic}

\usepackage{epsf,psfrag,amsfonts, xcolor,array}
\usepackage{amsmath,amsthm}
\usepackage{multirow}
\usepackage{cases}%[subnum]
\usepackage{verbatim}
\usepackage{xspace}
\usepackage{bbm}
\usepackage{mathrsfs}
\usepackage{setspace}
\usepackage[fleqn,tbtags]{mathtools}
\usepackage{tikz}

\usetikzlibrary{positioning}

\usepackage[colorinlistoftodos, textwidth=17mm]{todonotes}
\usepackage{color}

\usepackage[explicit]{titlesec}
\titlespacing{\paragraph}{0pt}{0.5em}{0.5em}[]
\author{Anima Anandkumar\thanks{University of California, Irvine. Email: a.anandkumar@uci.edu.} \and Rong Ge\footnote{Duke University. Email: rongge@cs.duke.edu}}

\newcommand\inner[1]{\ensuremath{\langle #1 \rangle}}

\definecolor{blue1}{HTML}{0066FF}
\definecolor{lpurple}{cmyk}{.05,0.18,0,0}

\def\tha{{\mbox{\tiny th}}}

\DeclareMathOperator{\Var}{Var}

 \def\0{{\bf 0}}

\def\qed{\hfill\hbox{${\vcenter{\vbox{
    \hrule height 0.4pt\hbox{\vrule width 0.4pt height 6pt
    \kern5pt\vrule width 0.4pt}\hrule height 0.4pt}}}$}}

\definecolor{myred}{rgb}{0.3,0.0,0.7}
\definecolor{dkg}{rgb}{0.1,0.7,0.2}
\definecolor{dkb}{rgb}{0.0,0.2,0.8}

\def\Lc{{\cal L}}

\def\Pc{{\cal P}}

\def\Sc{{\cal S}}

\def\Ebb{{\mathbb E}}

\def\Rbb{{\mathbb R}}

\newcommand{\bprf}{\begin{myproof}}
\newcommand{\eprf}{\end{myproof}}
\newcommand{\bp}{\begin{psfrags}}
\newcommand{\ep}{\end{psfrags}}
\newcommand{\bl}{\begin{lemma}}
\newcommand{\el}{\end{lemma}}
\newcommand{\bt}{\begin{theorem}}
\newcommand{\et}{\end{theorem}}
\newcommand{\bc}{\begin{center}}
\newcommand{\ec}{\end{center}}
\newcommand{\bi}{\begin{itemize}}
\newcommand{\ei}{\end{itemize}}
\newcommand{\ben}{\begin{enumerate}}
\newcommand{\een}{\end{enumerate}}
\newcommand{\bd}{\begin{definition}}
\newcommand{\ed}{\end{definition}}
\def\beq{\begin{equation}}
\def\eeq{\end{equation}\noindent}
\def\beqn{\begin{eqnarray}}
\def\eeqn{\end{eqnarray} \noindent}
\def\beqnn{  \begin{eqnarray*}}
\def\eeqnn{\end{eqnarray*}  \noindent}
\def\bcase{  \begin{numcases}}
\def\ecase{\end{numcases}   \noindent}
\def\bsbcase{  \begin{subnumcases}}
\def\esbcase{\end{subnumcases}   \noindent}

\newtheorem{theorem}{Theorem}
\newtheorem{corollary}{Corollary}
\newtheorem{lemma}{Lemma}
\newtheorem{assumption}{Assumption}
\newtheorem{claim}{Claim}

\theoremstyle{definition}
\newtheorem{definition}{Definition}
\theoremstyle{remark}%
\newtheorem{remark}{Remark}
\newenvironment{myproof}{\noindent{\em Proof:} \hspace*{1em}}{
    \hspace*{\fill} $\Box$ }
\newenvironment{proof_of}[1]{\noindent {\em Proof of #1: }}{\hspace*{\fill} $\Box$ }

\newcommand{\matplottc}[1]{               % single matlab plot twocolumn
        \unitlength .45truein
        \begin{center}
        \includegraphics{#1.ps}
        \end{picture}
        \end{center}
}

\def\psfancypar#1#2{\begingroup\def\par{\endgraf\endgroup\lineskiplimit=0pt}
               \setbox2=\hbox{\large\sc #2}
               \newdimen\tmpht \tmpht \ht2 \advance\tmpht by \baselineskip
               \font\hhuge=Times-Bold at \tmpht
               \setbox1=\hbox{{\hhuge #1}}
               \count7=\tmpht \count8=\ht1
               \divide\count8 by 1000 \divide\count7 by \count8
               \tmpht=.001\tmpht\multiply\tmpht by \count7
               \font\hhuge=Times-Bold at \tmpht
               \setbox1=\hbox{{\hhuge #1}}
               \noindent
                \hangindent1.05\wd1
               \hangafter=-2 {\hskip-\hangindent
               \lower1\ht1\hbox{\raise1.0\ht2\copy1}%
                \kern-0\wd1}\copy2\lineskiplimit=-1000pt}

\def\Kout{\setbox1=\hbox{\Huge\bf K}\hbox to
1.05\wd1{\hspace{.05\wd1}
}}
\def\Sout{\setbox1=\hbox{\Huge\bf S}\hbox to 1.05\wd1{\hspace{.05\wd1}
}}

\begin{document}

\title{Efficient approaches for escaping higher order saddle points  in non-convex optimization}

\maketitle

\begin{abstract}%   <- trailing '%' for backward compatibility of .sty file
Local search heuristics for non-convex optimizations are popular   in applied machine learning. However, in general it is  hard to  guarantee that such algorithms even  converge to a {\em local minimum}, due to the existence of complicated saddle point structures in high dimensions. Many functions have {\em degenerate} saddle points such that the first and second order derivatives cannot distinguish them with local optima.  In this paper we use higher order derivatives to escape these saddle points: we design the first efficient algorithm  guaranteed to converge to a third order local optimum (while existing techniques are at most second order). We also show that it is NP-hard to extend this further to finding fourth order local optima.
\end{abstract}

%\begin{keywords}
%  Non-convex optimization, degenerate saddle points, higher order conditions for local optimality,  trust region methods.
% \end{keywords}

\section{Introduction}

Recent trend in applied machine learning   has been dominated by the use of   large-scale non-convex optimization, e.g. deep learning. However, analyzing non-convex optimization in high dimensions   is very challenging.  Current theoretical results are  mostly negative regarding the hardness of reaching the globally optimal solution.

Less attention is paid to the issue of reaching a locally optimal solution. In fact, even this is computationally hard in the worst case~\citep{nie2015hierarchy}. The hardness arises due to  diversity and  ubiquity  of  critical points in high dimensions.  In addition to local optima, the set of critical points also consists of  saddle points, which    possess directions along which   the objective value improves. Since the objective function can be arbitrarily bad   at these points, it is important to develop strategies to escape them, in order to reach a local optimum.

The problem of saddle points is compounded in high dimensions. Due to curse of dimensionality, the number of saddle points   grows exponentially for many problems of interest, e.g.~\citep{auer1996exponentially,cartwright2013number,auffinger2013complexity}. Ordinary gradient descent can be stuck in a saddle point for an arbitrarily long time before making progress.  A few recent works have addressed this issue, either by incorporating second order Hessian information~\citep{nesterov2006cubic} or through  noisy stochastic gradient descent~\citep{ge2015escaping}. These works however require  the Hessian matrix at   the saddle point  to have a strictly negative eigenvalue, termed as  the {\em strict saddle} condition.  The  time to escape the saddle point   depends  (polynomially) on this negative eigenvalue. Some structured problems such as complete dictionary learning, phase retrieval and orthogonal tensor decomposition possess this property~\citep{sun2015nonconvex}.

On the other hand, for problems without the strict saddle property, the above techniques  can converge to a saddle point, which is disguised as   a local minimum when only first and second order information is used. We address this problem in this work, and extend the notion of second order optimality to higher order optimality conditions. We propose a new efficient algorithm that is guaranteed to converge to a third order local minimum, and show that it is NP-hard to find a fourth order local minimum.

Our results are relevant for a wide range of non-convex problems which   possess degenerate critical points. At these points, the   Hessian matrix is singular. Such points   arise  due to  symmetries in the optimization problem, e.g., permutation symmetry in a multi-layer neural network. Singularities also arise in over-specified models, where   the model capacity (such as the number of neurons in neural networks) exceeds the complexity of the target function. Here,  certain neurons can be eliminated (i.e. have weights set to zero), and such critical points  possess the so-called {\em elimination singularity}~\citep{wei2008dynamics}. Alternatively, two neurons can have the same weight, and this is known as {\em overlap singularity}~\citep{wei2008dynamics}. The Hessian matrix is singular at such   critical points.  This behavior is limited  not just to neural networks, but has  also been studied in overspecified   Gaussian mixtures,  radial basis function networks, ARMA models of time series~\citep{amari2006singularities,wei2008dynamics}, and student-teacher networks, also known as {\em soft committee models}~\citep{saad1995line,inoue2003line}.

\begin{figure}
\centering
\includegraphics[width=\textwidth]{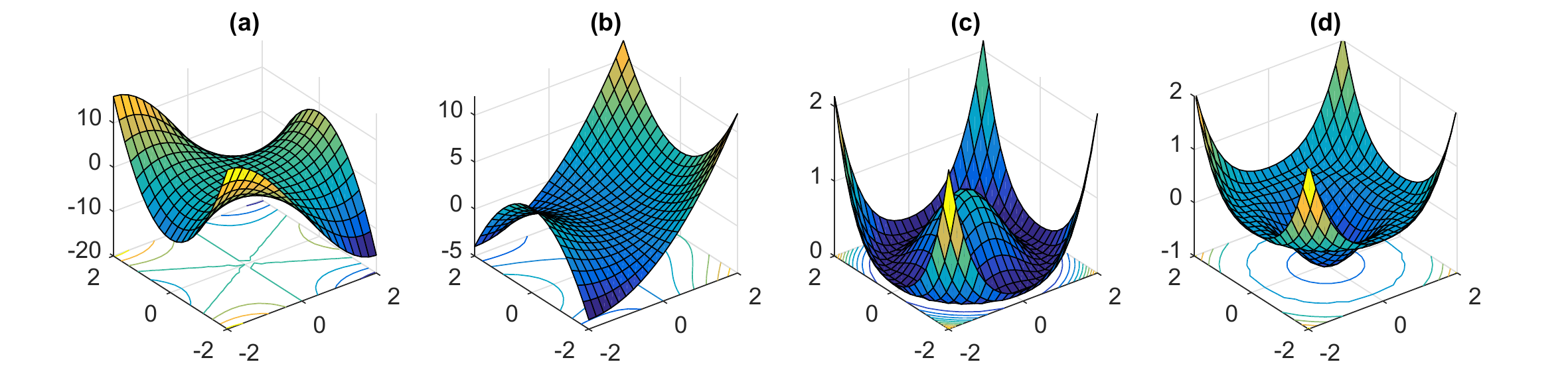}
\caption{Examples of Degenerate Saddle Points: (a) Monkey Saddle $-3x^2y+y^3$, $(0,0)$ is a second order local minimum but not third order local minimum; (b) $x^2y+y^2$, $(0,0)$ is a third order local minimum but not fourth order local minimum; (c) ``wine bottle'', the bottom of the bottle is a connected set with degenerate Hessian; (d) ``inverted wine bottle'': the points on the circle with degenerate Hessian are actually saddle points and not local minima.}\label{fig:functions}
\end{figure}

The current trend in practice is to incorporate overspecified models~\citep{giles2001overfitting}.
Theoretically, bad local optima are guaranteed to disappear in neural networks under massive levels of overspecification~\citep{safran2015quality}. On the other hand, as discussed above, the saddle point problem is compounded in these overspecified models. Empirically, the presence of singular saddle points is found to  slow down learning substantially~\citep{saad1995line,inoue2003line,amari2006singularities,wei2008dynamics}. Intuitively, these singular saddle points are    surrounded by    {\em plateaus} or flat regions  with a   sub-optimal objective value.
For these regions neither the gradient or Hessian information can lead to a direction that improves the function value. Therefore they can ``fool'' the (ordinary) first and second order algorithms and they may stuck there for long periods of time.
%The presence of the plateau can ``fool'' (ordinary) gradient descent into mistaking the point for a local minimum, and it can  get  stuck there for long periods of time.
%This problem is especially serious when the singular Hessian matrix has no (substantial) negative eigenvalues. In this case, we cannot obtain any direction of escape using only Hessian information.
Higher order derivatives are needed to classify the point as either a local optimum or a saddle point. In this work, we tackle this challenging problem of escaping such higher order saddle points.

%For instance,~\citet{amari2002geometrical} characterizes critical points in the overspecified setting, where there are more neurons than needed to fit the target function. They show that this leads to an explosion of critical points, and creates plateaus of connected points, which is hard to escape from.  It has also been observed in the context   of training . These are student neural networks which are of smaller size but try to mimic the behavior of the more powerful teacher networks.

\subsection{Summary of Results}

We call a point $x$ a $p^{\tha}$ order local minimum if for any nearby point $y$ $f(x) - f(y) \le o(\|x-y\|^p)$ (see Definition~\ref{def:pthorder}).

We give a necessary and sufficient condition for a point $x$ to be a third order local minimum (see Section~\ref{sec:condition}). Similar conditions (for even higher order) have been discussed in previous works, however their algorithmic implications were not known. We design an algorithm that is guaranteed to find a third order local minimum.

\begin{theorem} (Informal) There is an algorithm that always converges to a third order local minimum (see Theorem~\ref{thm:thirdlimit}). Also, in polynomial time the algorithm can find a point that is ``similar'' to a third order local minimum (see Theorem~\ref{thm:thirdfinite}).
\end{theorem}

By ``similar'' we mean the point $x$ approximately satisfies the necessary and sufficient condition for third order local minimum (see Definition~\ref{cond:third}): the gradient $\nabla f(x)$ is small, Hessian $\nabla^2 f(x)$ is almost positive semidefinite (p.s.d) and in every subspace where the Hessian is small, the norm of the third order derivatives is also small.

To the best of our knowledge this is the first algorithm that is guaranteed to converge to a third order local minimum. The algorithm alternates between a second order step (which we use cubic regularization\citep{nesterov2006cubic}) and a third order step. The third order step first identifies a ``competitive subspace'' where the third order derivative  has a much larger norm than the second order. It  then tries to find a good direction in this subspace to make improvement. For more details see Section~\ref{sec:alg}.

We also show that it is NP-hard to find a fourth order local minimum:

\begin{theorem} (Informal) Even for a well-behaved function, it is NP-hard to find a fourth order local minimum (see Theorem~\ref{thm:hard}).
\end{theorem}

%Critical points or stationary points are locations where the gradient $\nabla_x f(x)$ vanishes. Among them, local minimizers satisfy the second order necessary  condition that the Hessian matrix is positive semi-definite  $\nabla^2 f(x) \succeq 0$. When the  $\nabla^2 f(x) \succ  0$ it is also a sufficient condition for local minimality. When the Hessian $\nabla^2 f(x)\succeq 0$ is singular, higher order derivative information is needed. In this paper, we characterize necessary conditions for third order optimality.

\subsection{Related Work}

A popular approach to overcoming saddle points is to incorporate second order information.  However, the popular second order approach of  Newton's method is not suitable since it  converges to an arbitrary critical point, and does not distinguish between a local minimum and a saddle point. Directions along negative values of the Hessian matrix help in escaping the saddle point.    A simple solution is then to use these directions, whenever gradient descent improvements are small (which signals the approach towards a critical point)~\citep{frieze1996learning,vempala2011structure}.

A more elegant framework is the so-called trust region method~\citep{dauphin2014identifying,sun2015nonconvex} which involves optimizing the second order Taylor's approximation of the objective function in a local neighborhood of the current point. Intuitively, this  objective ``switches'' smoothly between first order and second order updates.~\citet{nesterov2006cubic}   propose  adding a  cubic regularization term to this Taylor's approximation. In a beautiful result, they show  that in each step, this  cubic regularized objective can be solved optimally due to hidden convexity and overall, the algorithm converges to a local optimum in bounded time. We give an overview of this algorithm in Section~\ref{sec:overview}. \cite{baes2009estimate} generalizes this idea to use higher order Taylor expansion, however the optimization problem is intractable even for third order Taylor expansion with quartic regularizer. \citet{ge2015escaping}  recently showed that     it is possible to escape saddle points using only first order information based on noisy stochastic gradient descent (SGD) in polynomial time in high dimensions.  \citet{lee2016gradient} showed that even without adding noise, in the limit gradient descent converges to (second order) local minimum with random initialization. In many applications, these first-order algorithms are far cheaper than the  computation of the Hessian eigenvectors.~\citet{nie2015hierarchy} propose using the hierarchy of semi-definite relaxations to compute  all the local optima which satisfy first and second order necessary conditions based on semi-definite relaxations.

All the above works deal with local optimality based on second order conditions. When the Hessian matrix is singular and p.s.d., higher order derivatives are required to determine whether it is a local optimum or a saddle point.  Higher order optimality conditions, both necessary and sufficient, have been characterized before, e.g.~\citep{bernstein1984systematic,warga1986higher}. But these conditions are not efficiently computable, and it is NP-hard to determine local optimality, given such  information about higher order  derivatives~\citep{murty1987some}.

%checking nonnegativity of quartic forms is NP-hard~\cite{lasserre2001global,nesterov2000squared}. 

\section{Preliminaries}

In this section we first introduce the classifications of saddle points. Next, as we often work with third order derivatives, and we treat it as a order 3 tensor, we introduce the necessary notations for tensors. %Finally we briefly review \cite{nesterov2006cubic}, which is a classical second order optimization algorithm.

\subsection{Critical Points}

Throughout the paper we consider functions $f:\Rbb^n \to \Rbb$ whose first three order derivatives exist. We represent the derivatives by $\nabla f(x) \in \Rbb^n$, $\nabla^2 f(x) \in \Rbb^{n\times n}$ and $\nabla^3 f(x)\in \Rbb^{n^3}$, where
$$
[\nabla f(x)]_i  = \frac{\partial}{\partial x_i} f(x),
[\nabla^2 f(x)]_{i,j}  = \frac{\partial^2}{\partial x_i\partial x_j} f(x),
[\nabla^3 f(x)]_{i,j,k}  = \frac{\partial^3}{\partial x_i\partial x_j\partial x_k} f(x).
$$

For such smooth function $f(x)$, we say $x$ is a {\em critical point} if $\nabla f(x) = \vec{0}$. Traditionally, critical points are classified into four cases according to the Hessian matrix:
\begin{enumerate}
\setlength{\itemsep}{0pt}
\item (Local Minimum) All eigenvalues of $\nabla^2 f(x)$ are positive.
\item (Local Maximum) All eigenvalues of $\nabla^2 f(x)$ are negative.
\item (Strict saddle) $\nabla^2 f(x)$ has at least one positive and one negative eigenvalues.
\item (Degenerate) $\nabla^2 f(x)$ has either nonnegative or nonpositive eigenvalues, with some eigenvalues equal to 0.
\end{enumerate}

As we shall see later in Section~\ref{sec:overview}, for the first three cases second order algorithms can either find a direction to reduce the function value (in case of local maximum or strict saddle), or correct asserting that the current point is a local minimum. However, second order algorithms cannot handle degenerate saddle points.
%Classification of critical points based on Hessian information: (i) all eigenvalues strictly positive: local minimum,
%(ii) all eigenvalues strictly positive: local maximum, (iii) strictly positive or negative eigenvalues: strict saddle point, (iv) degenerate Hessian with non-negative eigenvalues: degenerate saddle.

Degeneracy of Hessian   indicates the presence of a {\em gutter} structure, where a set of connected points all have the same value, and all are local minima, maxima or saddle points~\citep{dauphin2014identifying}. See for example Figure~\ref{fig:functions} (c) (d).%\aacomment{can we generate monkey saddle and wine bottle structures similar to Surya's paper?}

If the Hessian at a critical point $x$ is p.s.d., even if it has 0 eigenvalues we can say the point is a second order local minimum: for any $y$ that is sufficiently close to $x$, we have $f(x) - f(y) = o(\|x-y\|^2)$. That is, although there might be a vector $y$ that makes the function value decrease, the amount of decrease is a lower order term compared to $\|x-y\|^2$. In this paper we consider higher order local minimum:

\begin{definition}[$p$-th order local minimum]\label{def:pthorder}A critical point $x$ is a $p$-th order local minimum, if there exists constants $C,\epsilon> 0$ such that for every $y$ with $\|y-x\|\le \epsilon$,
$$
f(y) \ge f(x) - C \|x-y\|^{p+1}.
$$
\end{definition}

Every critical point is a first order local minimum, and every point that satisfies the second order necessary condition ($\nabla f(x) = 0, \nabla^2 f(x)\succeq 0$) is a second order local minimum.

\subsection{Matrix and Tensor Notations}

For a vector $v \in \Rbb^n$, we use $\|v\|$ to denote its $\ell_2$ norm. 
For a matrix $M \in \Rbb^{n\times n}$, we use $\|M\|$ to denote its spectral (operator) norm. All the matrices we consider are symmetric matrices, and they can be decomposed using eigen-decomposition:
$$
M = \sum_{i=1}^n \lambda_i v_i v_i^\top.
$$
In this decomposition $v_i$'s are orthonormal vectors, and $\lambda_i$'s are the eigenvalues of $M$. We always assume $\lambda_1\ge \lambda_2 \ge \ldots \ge \lambda_n$. We use $\lambda_1(M)$ to denote its largest eigenvalue and $\lambda_n(M)$ to denote its smallest eigenvalue. By the property of symmetric matrices we also know $\|M\| = \max\{|\lambda_1(M)|, |\lambda_n(M)|\}$. We use $\|M\|_F$ to denote the Frobenius norm of the matrix $\|M\|_F = \sqrt{\sum_{i,j\in [n]} M_{i,j}^2}$.

The third order derivative is represented by a $n\times n\times n$ tensor $T$. We use the following multilinear notation to simplify the notations of tensors:
\begin{definition}[Multilinear notations] Let $T\in \Rbb^{n\times n\times n}$ be a third order tensor. Let $U\in \Rbb^{n\times n_1}$, $V\in \Rbb{n\times n_2}$ and $W\in \Rbb^{n\times n_3}$ be three matrices, then the multilinear form $T(U,V,W)$ is a tensor in $\Rbb^{n_1\otimes n_2\otimes n_3}$ that is equal to
$$
[T(U,V,W)]_{p,q,r} = \sum_{i,j,k\in [n]} T_{i,j,k} U_{i,p}V_{j,q}W_{k,r}.
$$
\end{definition}

In particular, for vectors $u,v,w\in \Rbb^n$, $T(u,v,w)$ is a number that relates linearly in $u,v$ and $w$ (similar to $u^\top M v$ for a matrix); $T(u,v,I)$ is a vector in $\Rbb^n$ (similar to $Mu$ for a matrix); $T(u,I,I)$ is a matrix in $\Rbb^{n\times n}$. 

The Frobenius norm of a tensor $T$ is defined similarly as matrices: $\|T\|_F = \sqrt{\sum_{i,j,k\in [n]} T_{i,j,k}^2}$. The spectral norm (also called injective norm) of a tensor is defined as $$\|T\| = \max_{\|u\| = 1, \|v\| =1, \|w\| = 1} T(u,v,w).$$ 

We say a tensor is symmetric if $T_{i,j,k} = T_{\pi(i,j,k)}$ for any permutation of the indices. For symmetric tensors the spectral norm is also equal to $\|T\| = \max_{\|u\| = 1} T(u,u,u)$. In both cases it is NP-hard to compute the spectral norm of a tensor\citep{hillar2013most}.

We will often need to project a tensor $T$ to a subspace $\Pc$. Let $P$ be the projection matrix to the subspace $P$, we use the notation $\mbox{Proj}_\Pc T$ which denotes $T(P,P,P)$. Intuitively, $[T(P,P,P)]_{u,v,w} = T(Pu, Pv, Pw)$, that is, the projected tensor applied to vector $u,v,w$ is equivalent to the original tensor applied to the projection of $u,v,w$.

\section{Overview of Nestorov's Cubic Regularization}\label{sec:overview}

In this section we review the guarantees of Nesterov's Cubic Regularization algorithm\citep{nesterov2006cubic}. We will use this algorithm as a key step later in Section~\ref{sec:alg}, and prove analogous results for third order local minimum.

The algorithm requires the first two order derivatives exist and the following smoothness constraint:

\begin{assumption}[Lipschitz-Hessian]\label{assump:lipschitzhessian}
$$
\forall x,y, \|\nabla^2 f(x) - \nabla^2 f(y)\| \le R\|x-y\|.
$$
\end{assumption}

At a point $x$, the algorithm tries to find a nearby point $z$ that optimizes the degree two Taylor's expansion: $f(x) + \inner{\nabla f(x), z-x} + \frac{1}{2}(z-x)^\top(\nabla^2 f(x))(z-x)$, with the cubic distance $\frac{R}{6}\|z-x\|^3$ as a regularizer. See Algorithm~\ref{alg:cubic} for one iteration of the algorithm. The final algorithm generates a sequence of points $x^{(0)},x^{(1)},x^{(2)}, \ldots$ where $x^{(i+1)} = \mbox{CubicReg}(x^{(i)})$.

\begin{algorithm}
\begin{algorithmic}
\REQUIRE function $f$, current point $x$, Hessian smoothness $R$
\ENSURE Next point $z$ that satisfies Theorem~\ref{thm:cubicreg}.
\STATE Let $z = \arg\min f(x) + \inner{\nabla f(x), z-x} + \frac{1}{2}(z-x)^\top(\nabla^2 f(x))(z-x) + \frac{R}{6}\|z-x\|^3$.
\RETURN $z$
\end{algorithmic}
\caption{CubicReg\citep{nesterov2006cubic}}\label{alg:cubic}
\end{algorithm}

The optimization problem that Algorithm~\ref{alg:cubic} tries to solve may seem difficult, as it has a cubic regularizer $\|z-x\|^3$. However, \cite{nesterov2006cubic} showed that it is possible to solve this optimization problem in polynomial time. 

For each point, define $\mu(z)$ to measure how close the point $z$ is to satisfying the second order optimality condition:

\begin{definition} \label{def:mu}
$\mu(z) = \max\left\{\sqrt{\frac{1}{R}\|\nabla f(z)\|}, -\frac{2}{3R}\lambda_n \nabla^2 f(z)\right\}$
\end{definition}

When $\mu(z) = 0$ we know $\nabla f(z) = 0$ and $\nabla^2 f(z) \succeq 0$, which satisfies the second order necessary conditions (and in fact implies that $z$ is a second order local minimum). When $\mu(z)$ is small we can say that the point $z$ approximately satisfies the second order optimality condition.

For one step of the algorithm the following guarantees can be proven\footnote{All of guarantees we stated here correspond to setting the regularizer $R$ to be exactly equal to the smoothness in Assumption~\ref{assump:lipschitzhessian}.}

\begin{theorem}\label{thm:cubicreg} \citep{nesterov2006cubic}
Suppose $z = \mbox{CubicRegularize}(x)$, then $\|z-x\| \ge \mu(z)$ and $f(z) \le f(x) - R\|z-x\|^3/12$.
\end{theorem}

Using Theorem~\ref{thm:cubicreg}, \cite{nesterov2006cubic} can get strong convergence results for the sequence $x^{(0)},x^{(1)},x^{(2)}, \ldots$

\begin{theorem}\label{thm:cubicreg2}\citep{nesterov2006cubic}
If $f(x)$ is bounded below by $f(x^*)$, then $\lim_{i\to \infty} \mu(x^{(i)}) = 0$, and for any $t \ge 1$ we have
$$
\min_{1\le i\le t} \mu(x^{(i)}) \le \frac{8}{3}\cdot \left(\frac{3(f(x^{(0)}) - f(x^*))}{2t R}\right)^{1/3}.
$$
\end{theorem}

This theorem shows that within first $t$ iterations, we can find a point that ``looks similar'' to a second order local minimum in the sense that gradient is small and Hessian does not have a negative eigenvalue with large absolute value. It is also possible to prove stronger guarantees for the limit points of the sequence:

\begin{theorem}\label{thm:cubicreg3}\citep{nesterov2006cubic}
If the level set $\Lc(x^{(0)}) := \{x|f(x) \le f(x^{(0)})\}$ is bounded, then the following limit exists
$$
\lim_{i\to\infty} f(x^{(i)}) = f^*,
$$
The set $X^*$ of the limit points of this sequence is non-empty. Moreover this is a connected set such that for any $x\in X^*$ we have
$$
f(x) = f^*, \nabla f(x) = \vec{0}, \nabla^2 f(x)\succeq 0.
$$
\end{theorem}

Therefore the algorithm always converges to a set of points that are all second order local minima.
\section{Third Order Necessary Condition}
\label{sec:condition}
In this section we present a condition for a point to be a third order local minimum, and show that it is necessary and sufficient for a class of smooth functions. Proofs are deferred to Appendix~\ref{app:third}.

All the functions we consider satisfies the following natural smoothness conditions

\begin{assumption} [Lipschitz third Order] We assume the first three derivatives of $f(x)$ exist, and for any $x,y\in \Rbb^n$,
$$
\|\nabla^3 f(x) - \nabla^3 f(y)\|_F \le L\|x-y\|.
$$\label{assump:lipthird}
\end{assumption}
Under this assumption, we state our conditions for a point to be a third order local minimum.

\begin{definition}[Third-order necessary condition]\label{cond:third} 
A point $x$ satisfy third-order necessary condition, if
\begin{enumerate}
\item $\nabla f(x) = 0$.
\item $\nabla^2 f(x) \succeq 0$.
\item For any $u$ that satisfy $u^\top (\nabla^2 f(x))u = 0$, $[\nabla^3 f(x)](u,u,u) = 0$.
\end{enumerate}
\end{definition}

We first note that this condition can be verified in polynomial time.

\begin{claim}
Conditions in Definition~\ref{cond:third} can be verified in polynomial time given the gradients $\nabla f(x), \nabla^2 f(x)$ and $\nabla^3 f(x)$.
\end{claim}

\begin{proof}
It is easy to check whether $\nabla f(x) = 0$ and $\nabla^2 f(x) \succeq 0$. We can also use SVD to compute the subspace $\Pc$ such that $u^\top (\nabla^2 f(x)) u = 0$ if and only if $u\in \Pc$.

Now we can compute the projection of $\nabla^3 f(x)$ in the subspace $\Pc$, and we claim the third condition is violated if and only if the projection is nonzero.

If the projection is zero, then clearly $[\nabla^3 f(x)](u,u,u)$ is 0 for any $u\in \Pc$. On the other hand, if projection $Z$ is nonzero, let $u$ be a uniform Gaussian vector that has unit variance in all directions of $u$, then we know $\Ebb[[[\nabla^3 f(x)](u,u,u)]^2] \ge \|Z\|_F^2 > 0$, so there must exists an $u\in \Pc$ such that $[\nabla^3 f(x)](u,u,u) \ne 0$.
\end{proof}

\begin{theorem}\label{thm:thirdcondition}
Given a function $f$ that satisfies Assumption~\ref{assump:lipthird}, a point $x$ is third order optimal if and only if it satisfies Condition~\ref{cond:third}.
\end{theorem}

Before proving the theorem, we first show a bound on $f(y)$ and a Taylor's expansion of $f$ at point $x$.

\begin{lemma}
\label{lem:lipbound}
For any $x,y$, we have
$$
|f(y) - f(x) - \inner{\nabla f(x), y-x} + \frac{1}{2}(y-x)^\top \nabla^2 f(x) (y-x) - \frac{1}{6} \nabla^3 f(x)(y-x, y-x, y-x)| \le \frac{L}{24}\|y-x\|^4.
$$
\end{lemma}

The Lemma can be proved by integrating over the third order derivatives three times and bounding the differences. Details are deferred to Appendix~\ref{app:third}.

This lemmas allow us to ignore the fourth order term $\|y-x\|^4$ and focus on the order 3 Taylor expansion when $\|y-x\|$ is small. To prove Theorem~\ref{thm:thirdcondition}, intuitively, the ``only if'' direction (local minimum to necessary condition) is easy because if any condition in Definition~\ref{cond:third} is violated, we can use that particular derivative to find a direction that improves the function value. For the ``if'' direction (necessary condition to third order local minimum), the main challenge is to balance the contribution we get from the positive part of the Hessian matrix and the third order derivatives. For details see Appendix~\ref{app:third}.

\section{Algorithm for Finding Third Order Optimal Points}
\label{sec:alg}
%\aacomment{would be good to number and list all the assumptions}

We design an algorithm that is guaranteed to converge to a third order local minimum. Throughout this section we assume both Assumptions~\ref{assump:lipschitzhessian} and \ref{assump:lipthird} \footnote{Note that we actually only cares about a {\em level set} $\Lc = \{x|f(x) \le f(x^{(0)})\}$, as long as this set is bounded Assumptions~\ref{assump:lipschitzhessian} follows from Assumption~\ref{assump:lipthird}}.

The main intuition of the algorithm is similar to the proof of Theorem~\ref{thm:thirdcondition}: the algorithm tries to make improvements using first, second or third order information. However, the nature of the third order condition makes it challenging for the algorithm to guarantee progress.

Consider a potential local minimum point $x$. It is very easy to check whether $\nabla f(x) \ne  0$ or $\lambda_{min} (\nabla^2 f(x) )< 0$, and to make progress using the corresponding directions. However, to verify Condition 3 in Definition~\ref{cond:third}, we need to do it in the right subspace.

The na\"ive guess is that we should take the eigensubspace of $\nabla^2 f(x)$ with eigenvalue at most 0. However, this is not correct because even if $x$ is a second order local minimum that does not satisfy the third order condition, it is still possible to have a sequence of $x^{(i)}$'s that converge to $x$ with $\nabla^2 f(x^{(i)})$ all be {\em strictly} positive definite. Hence all the $x^{(i)}$'s appear to satisfy Condition 3 in Definition~\ref{cond:third}. We do not want to the algorithm to spend too much time around this point $x$, so we need to identify a subspace that may have some positive eigenvalues. In order to make sure we can find a vector the contribution from third order term is larger than the second order term, we define {\em competitive subspace} below:

%The algorithm tries to discover a subspace where the third order term is {\em competitive} against the second order term (which roughly means we can choose a vector where the contribution from the third order term is larger than the second order term). We make the definition precise.

\begin{definition}[eigensubspace] For any symmetric matrix $M$, let its eigendecomposition be $M = \sum_{i=1}^n \lambda_i v_i v_i^\top$ (where $\lambda_i$'s are eigenvalues and $\|v_i\| = 1$), we use $\Sc_{\tau}(M)$ to denote the span of eigenvectors with eigenvalue at most $\tau$. That is
$$
\Sc_{\tau}(M) = \mbox{span}\{v_i|\lambda_i \le \tau\}.
$$
\end{definition}

\begin{definition}[competitive subspace]\label{def:competitive} For any $Q > 0$, and any point $z$, let the competitive subspace $\Sc(z)$ be the largest eigensubspace $\Sc_\tau(\nabla^2 f(z))$, such that if we let $C_Q(z)$ be the norm of the third order derivatives in this subspace
$$
C_Q(z) = \|\mbox{Proj}_{\Sc(z)} \nabla^3 f(z)\|_F,
$$
then $\tau \le C_Q^2/12LQ^2$.

If no such subspace exists then let $\Sc(z)$ be empty and $C_Q(z) = 0$.
\end{definition}

Similar to $\mu(z)$ as in Definition~\ref{def:mu}, $C_Q(z)$ can be viewed as how Condition 3 in Definition~\ref{cond:third} is satisfied approximately. If both $\mu(z)$ and $C_Q(z)$ are $0$ then the point $z$ satisfies third order necessary conditions.

Intuitively, competitive subspace is a subspace where the eigenvalues of the Hessian are small, but the Frobenius norm of the third order derivative is large. Therefore we are likely to make progress using the third order information. The parameters in Definition~\ref{def:competitive} are set so that if there is a unit vector $u\in \Sc(z)$ such that $[\nabla^3 f(z)](u,u,u) \ge \|\mbox{Proj}_{\Sc(z)} \nabla^3 f(z)\|_F/Q$ (see Theorem~\ref{thm:approx}), then we can find a new point where the sum of second, third and fourth order term can be bounded (see Lemma~\ref{lem:thirdorderstep}).

\begin{remark}
The competitive subspace in Definition~\ref{def:competitive} can be computed in polynomial time, see Algorithm~\ref{alg:compet}. The main idea is that we can compute the eigendecomposition  of the Hessian $\nabla^2 f(z) = \sum_{i=1}^n \lambda_i v_iv_i^\top$, and then there are only $n$ different subspaces ($\mbox{span}\{v_n\},\mbox{span}\{v_{n-1},v_n\},$ $\ldots,\mbox{span}\{v_1,v_2,\ldots v_n\}$).
We can enumerate over all of them, and check for which subspaces the norm of the third order derivative is large.
\end{remark}

Now we are ready to state the algorithm. The algorithm is a combination of the cubic regularization algorithm and a third order step that tries to use the third order derivative in order to improve the function value in the competitive subspace.

%In addition to the Lipschitz bound on the third order derivative. We need the following assumption that states Hessian is also Lipschitz:
%
%This is a technical condition that is mostly used to make the proof and intuitions clear. For all these conditions, we only care about the level set $\Lc(f(x_0)) = \{x|f(x)\le f(x_0)\}$. For most natural functions the level set is compact, and the Lipschitz bound on Hessian actually follows from the Lipschitz bound on third order derivative (and the condition at optimal $x^*$).

\begin{algorithm}
\begin{algorithmic}
\FOR{$i = 0$ \TO $t-1$}
\STATE $z^{(i)} = \mbox{CubicReg}(x^{(i)})$.
\STATE Let $\epsilon_1 = \|\nabla f(z^{(i)})\|$,
\STATE Let $\Sc(z), C_Q(z)$ be the competitive subpace of $f(z)$ (Definition~\ref{def:competitive}).
\IF {$C_Q(z)\ge Q(24\epsilon_1 L)^{1/3}$}
\STATE $u = \mbox{Approx}(\nabla^3 f(z^{(i)}), \Sc)$.
\STATE $x^{(i+1)} = z^{(i)}-\frac{C_Q(z)}{LQ}u$.
\ELSE
\STATE $x^{(i+1)} = z^{(i)}$.
\ENDIF
\ENDFOR
\end{algorithmic}
\caption{Third Order Optimization}\label{alg:main}
\end{algorithm}

Suppose we have the following approximation guarantee for Algorithm~\ref{alg:approx}

\begin{algorithm}
\begin{algorithmic}
\REQUIRE Tensor $T$, subspace $\Sc$.
\ENSURE unit vector $u\in \Sc$ such that $T(u,u,u) \ge \|\mbox{Proj}_\Sc T\|_F/Q$.
\REPEAT
\STATE Let $\hat{u}$ be a random standard Gaussian in subspace $\Sc$.
\STATE Let $u = \hat{u}$
\UNTIL $|T(u,u,u)| \ge \|\mbox{Proj}_\Sc T\|_F/Bn^{1.5}$ for a fixed constant $B$
\RETURN $u$ if $T(u,u,u)>0$ and $-u$ otherwise.
\end{algorithmic}
\caption{Approximate Tensor Norms}\label{alg:approx}
\end{algorithm}

\begin{theorem}\label{thm:approx}
There is a universal constant $B$ such that the expected number of iterations of Algorithm~\ref{alg:approx} is at most $2$, and the output of $\mbox{Approx}$ is a unit vector $u$ that satisfies $T(u,u,u) \ge \|\mbox{Proj}_\Sc T\|_F/Q$ for $Q = Bn^{1.5}$.
\end{theorem}

The proof of this theorem follows directly from anti-concentration (see Appendix~\ref{app:approx}. Notice that there are other algorithms that can potentially give better approximation (lower value of $Q$) which will improve the rate of our algorithm. However in this paper we do not try to optimize over dependencies over the dimension $n$, that is left as an open problem.

%\rgcomment{Proof should go to appendix, just say it follows from anti-concentration}

By the choice of the parameters in the algorithm, we can get the following guarantee (which is analogous to Theorem~\ref{thm:cubicreg}):

\begin{lemma}\label{lem:thirdorderstep}
If $C_Q(z) \ge Q(24\epsilon_1 L)^{1/3}$, $u$ is a unit vector in $\Sc(z)$ and $[\nabla^3 f(z)](u,u,u) \ge \|\mbox{Proj}_{\Sc(z)} \nabla^3 f(z)\|_F/Q$. Let $x' = z- C_Q(z)/LQ \cdot u$. then we have
$$
f(x') \le f(z) - \frac{C_Q(z)^4}{24L^3Q^4}.
$$
\end{lemma}

\begin{proof}
Let $\epsilon = C_Q(z)/LQ$, then by Lemma~\ref{lem:lipbound} we know
$$
f(x') \le f(z) - \frac{\epsilon^3C}{6Q} + \epsilon_1 \epsilon  + \epsilon_2 \epsilon^2/2 + L\epsilon^4/24.
$$

Here $\epsilon_1 = \|\nabla f(z)\|$, and $\epsilon_2 \le \frac{C_Q(z)^2}{12 LQ^2}$ by the construction of the subspace.

By the choice of parameters, we know the terms $\epsilon_1\epsilon, \epsilon_2 \epsilon^2/2, L\epsilon^4/24$ are all bounded by $\frac{\epsilon^3C_Q(z)}{24Q}$, therefore
$$
f(x') \le f(z) - \frac{\epsilon^3C_Q(z)}{24Q} =f(z) - \frac{C_Q(z)^4}{24L^3Q^4}
$$
\end{proof}

Using this Lemma, and Theorem~\ref{thm:cubicreg} for cubic regularization, we can show that both progress measure goes to 0 as the number of steps increase (this is analogous to Theorem~\ref{thm:cubicreg2}).

\begin{theorem}
\label{thm:thirdfinite}
Suppose the algorithm starts at $f(x_0)$, and $f$ has global min at $f(x^*)$. Then in one of the $t$ iterations we have
\begin{enumerate}
\item $\mu(z) \le \left(\frac{12(f(x_0)-f(x^*)}{Rt}\right)^{1/3}$.
\item $C_Q(z) \le \max\left\{Q(24\|\nabla f(z)\| L)^{1/3}, Q\left(\frac{24L^3(f(x_0)-f(x^*))}{t} \right)^{1/4}\right\}.$
\end{enumerate}
\end{theorem}

Recall $\mu(z) = \max\left\{\sqrt{\frac{1}{R}\|\nabla f(z)\|}, -\frac{2}{3R}\lambda_n \nabla^2 f(z)\right\}$ is intuitively measuring how much first and second order progress the algorithm can make. The value $C_Q(z)$ as defined in Definition~\ref{def:competitive} is a measure of how much third order progress the algorithm can make. The theorem shows both values goes to 0 as $t$ increases (note that even the first term $Q(24\|\nabla f(z)\| L)^{1/3}$ in the bound for $C_Q(z)$ goes to 0 because the $\|\nabla f(z)\|$ goes to 0). 

\begin{proof}
By the guarantees of Theorem~\ref{thm:cubicreg} and Lemma~\ref{lem:thirdorderstep}, we know the sequence of points $x^{(0)}, z^{(0)}, \ldots, x^{(i)}, z^{(i)}, \ldots$ has non-increasing function values. Also,
$$
\sum_{i=1}^t f(x^{(i)}) - f(x^{(i-1)}) \le f(x_0) - f(x^*).
$$
So there must be an iteration where $f(x^{(i)}) - f(x^{(i-1)}) \le \frac{f(x_0) - f(x^*)}{t}$.

If $\mu(z) > \left(\frac{12(f(x_0)-f(x^*)}{Rt}\right)^{1/3}$, then Theorem~\ref{thm:cubicreg} implies $f(x^{(i-1)}) - f(z^{(i-1)}) > \frac{f(x_0) - f(x^*)}{t}$, which is impossible.

On the other hand if $C_Q(z) \le \max\left\{Q(24\|\nabla f(z)\| L)^{1/3}, Q\left(\frac{24L^3(f(x_0)-f(x^*))}{t} \right)^{1/4}\right\}$, then the third order step makes progress, and we know $f(z^{(i-1)}) - f(x^{(i)}) > \frac{f(x_0) - f(x^*)}{t}$, which is again impossible.
\end{proof}

We can also show that when $t$ goes to infinity the algorithm converges to a third order local minimum (similar to Theorem~\ref{thm:cubicreg3}).

\begin{theorem}
\label{thm:thirdlimit}
When $t$ goes to infinity, the values $f(x^{(t)})$ converge. If the level set $\Lc(f(x_0)) = \{x|f(x)\le f(x_0)\}$ is compact, then the sequence of points $x^{(t)}, z^{(t)}$ has nonempty limit points, and every limit point $x$ satisfies the third order necessary conditions.
\end{theorem}

\begin{proof}
By Theorem~\ref{thm:cubicreg} and Lemma~\ref{lem:thirdorderstep}, we know the function value is non-increasing, and it has a lowerbound $f(x^*)$, so  the value must converge.

The existence of limit points is guaranteed by the compactness of the level set. The only thing left to prove is that every limit point $x$ must satisfy the third order necessary conditions.

Notice that $f(x^{(0)}) - \lim_{t\to \infty} f(x^{(t)} \ge \sum_{i=0}^\infty \frac{R\mu(z^{(i)})^3}{12}+\frac{C_Q(z^{(i)})^4}{24L^3 Q^4}$, so $\lim_{i\to \infty}\mu(z^{(i)} = 0$ and $\lim_{i\to\infty} C_Q(z^{(i)}) = 0$. Also we know further $\lim_{i\to\infty} \|z^{(i)} - x^{(i)}\| = 0$. Therefore wlog a limit point $x$ is also a limit point of sequence $z$, and $\lim_{i\to \infty} \|\nabla f(z)\| = 0$. Also we know $H = \nabla^2 f(x)$ is PSD, because otherwise points near $x$ will have nonzero $\mu(z^{(i)}$ and $x$ cannot be a limit point.

Now we only need to check the third order condition. Assume towards contradiction that third order condition is not true. The we know the Hessian has a subspace $\Pc$ with $0$ eigenvalues, and the third order derivative has norm at least $\epsilon$ in this subspace. By matrix perturbation theory, when $z$ is very close to $x$, $\Pc$ is very close to $\Sc_\epsilon(z)$ for $\epsilon\to 0$, on the other hand the third order tensor also converges to $\nabla^3 f(x)$ (by Lipschitz condition), so $\Sc_\epsilon(z)$ will eventually be a competitive subspace and $C_Q(z)$ is at least $\epsilon/2$ for all $z$. However this is impossible as $\lim_{i\to\infty} C_Q(z^{(i)}) = 0$.
\end{proof}

\begin{remark}
Note that not all third order local minimum can be the limit point for Algorithm~\ref{alg:main}. This is because if $f(x)$ has very large third order derivatives but relatively smaller Hessian, even though the Hessian might be positive definite (so $x$ is in fact a local minimum), Algorithm~\ref{alg:main} may still find a non-empty competitive subspace, and will be able to reduce the function value and escape from the saddle point. An example is for the function $f(x) = x^2 - 100x^3 + x^4$, $x = 0$ is a local minimum but the algorithm can escape from that and find the global minimum.
\end{remark}

In the most general case it is hard to get a convergence rate for the algorithm because the function may have higher order local minima. However, if the function has nice properties then it is possible to prove polynomial rates of convergence.

\begin{definition}[strict third order saddle] We say a function is strict third order saddle, if there exists constants $\alpha, c_1, c_2, c_3, c_4 > 0$ such that for any point $x$ one of the following is true:
\begin{enumerate}
\setlength{\itemsep}{0pt}
\item $\|\nabla f(x)\| \ge c_1$.
\item $\lambda_n(f(x)) \le -c_2$.
\item $C_Q(f(x)) \ge c_3$.
\item There is a local minimum $x^*$ such that $\|x-x^*\| \le c_4$ and the function is $\alpha$-strongly convex restricted to the region $\{x|\|x-x^*\| \le 2c_4\}$.
\end{enumerate}
\label{def:strictthirdsaddle}
\end{definition}

This is a generalization of the strict saddle functions defined in \cite{ge2015escaping}. Even if a function has degenerate saddle points, it may still satisfy this condition.

\begin{corollary}
When $t \ge \mbox{poly}(n, L, R, Q, f(x_0) - f(x^*)) \max\{(1/c_1)^{1.5}, (1/c_2)^3, (1/c_3)^{4.5}\}$, there must be a point $z^{(i)}$ with $i\le t$ that is in case 4 in Definition~\ref{def:strictthirdsaddle}.
\end{corollary}

\begin{proof}
We use $\tilde{O}$ to only focus on the polynomial dependency on $t$ and ignore polynomial dependency on all other parameters.

By Theorem~\ref{thm:thirdfinite}, we know there must be a $z^{(i)}$ which satisfies $\mu(z^{(i)} \le \tilde{O}((1/t)^{1/3})$ and $C_Q(z) \le \tilde{O}(\max\{(1/t)^{1/4}, \|\nabla f(z)\|^{1/3}\})$.

By the Definition of $\mu$ (Definition~\ref{def:mu}), we know  $\|\nabla f(z)\| \le \tilde{O}(\mu(z))^2 = \tilde{O}(t^{-2/3})$, $\lambda_n(\nabla^2 f(z)) \ge -\tilde{O}(t^{-1/3})$.

Using the fact that $\|\nabla f(z)\| \le \tilde{O}(\mu(z))^2 = \tilde{O}(t^{-2/3}$, we know $$C_Q(z) \le \tilde{O}(\max\{(1/t)^{1/4}, \|\nabla f(z)\|^{1/3}\}) = \tilde{O}(t^{-2/9}).$$

Therefore, when $t \ge \mbox{poly}(n, L, R, Q, f(x_0) - f(x^*)) \max\{(1/c_1)^{1.5}, (1/c_2)^3, (1/c_3)^{4.5}\}$, the point $z$ must satisfy
\begin{enumerate}
\setlength{\itemsep}{0pt}
\item $\|\nabla f(z)\| < c_1$;
\item $\lambda_n (\nabla^2 f(z)) < -c_2$;
\item $C_Q(z) < c_3$.
\end{enumerate}

Therefore the first three cases in Definition~\ref{def:strictthirdsaddle} cannot happen and $z$ must be near a local minimum.
\end{proof}

\section{Hardness for Finding a fourth order Local Minimum}

In this section we show it is hard to find a fourth order local minimum even if the function we consider is very well-behaved.

\begin{definition}[Well-behaved function] We say a function $f$ is well-behaved if it is infinite-order differentiable, and satisfies:
\begin{enumerate}
\setlength{\itemsep}{0pt}
\item $f(x)$ has a global minimizer at some point $\|x\|\le 1$.
\item $f(x)$ has bounded first 5 derivatives for $\|x\| \le 1$.
\item For any direction $\|x\| = 1$, $f(tx)$ is increasing for $t \ge 1$.
\end{enumerate}
\end{definition}

Clearly, all local minimizers of a well-behaved function lies within the unit $\ell_2$ ball, and $f(x)$ is smooth with bounded derivatives within the unit $\ell_2$ ball. These functions also satisfy Assumptions~\ref{assump:lipschitzhessian} and \ref{assump:lipthird}. All the algorithms mentioned in previous sections can work in this case and find a local minimum up to order 3. However, this is not possible for fourth order.

\begin{theorem}
\label{thm:hard}
It is NP-hard to find a fourth order local minimum of a function $f(x)$, even if $f$ is guaranteed to be well-behaved.
\end{theorem}

The main idea of the proof comes from the fact that we cannot even verify the nonnegativeness of a degree 4 polynomial (hence there are cases where we cannot verify whether a point is a fourth order local minimum or not).

\begin{theorem}\label{thm:hardverify}\cite{nesterov2000squared,hillar2013most} It is NP-hard to tell whether a degree 4 homogeneous polynomial $f(x)$ is nonnegative.
\end{theorem}

\begin{remark} The NP hardness for nonnegativeness of degree 4 polynomial has been proved has been proved in several ways. In \cite{nesterov2000squared} the reduction is from the SUBSET SUM problem, which results in a polynomial that can have exponentially large coefficients and does not rule out FPTAS. However, the reduction in \cite{hillar2013most} relies on the hardness of copositive matrices, which in turn depends on the hardness of INDEPENDENT SET\citep{dickinson2014computational}. This reduction gives a polynomial whose coefficients can be bounded by $\mbox{poly}(n)$, and a polynomial gap that rules out FPTAS.
\end{remark}

To prove Theorem~\ref{thm:hard} we only need to reduce the nonnegativeness problem in Theorem~\ref{thm:hardverify} to the problem of finding a fourth order local minimum. We can convert a degree 4 polynomial to a well behaved function by adding a degree 6 regularizer $\|x\|^6$. We shall show when the degree 4 polynomial is nonnegative the $\vec{0}$ point is the only fourth order local minimum; when the degree 4 polynomial has negative directions then every fourth order local minimum must have negative function value. The details are deferred to Section~\ref{app:hard}.

\section{Conclusion}

Complicated structures of saddle points are a major problem for optimization algorithms. In this paper we investigate the possibilities of using higher order derivatives in order to avoid degenerate saddle points. We give the first algorithm that is guaranteed to find a 3rd order local minimum, which can solve some problems caused by degenerate saddle points. However, we also show that the same ideas cannot be generalized to higher orders.

There are still many open problems related to degenerate saddle points and higher order optimization algorithms. Are there interesting class of functions that satisfies the strict 3rd order saddle property (Definition~\ref{def:strictthirdsaddle})? Can we design a 3rd order optimization algorithm for constrained optimization? We hope this paper inspires more research in these directions and eventually design efficient optimization algorithms whose performance do not suffer from degenerate saddle points.
\newpage
\bibliographystyle{plainnat}
\bibliography{bib}

\newpage
\appendix
\section{Omitted Proofs}

\subsection{Omitted Proofs in Section~\ref{sec:condition}}
\label{app:third}

\begin{lemma} [Lemma~\ref{lem:lipbound} Restated]
For any $x,y$, we have
$$
|f(y) - f(x) - \inner{\nabla f(x), y-x} + \frac{1}{2}(y-x)^\top \nabla^2 f(x) (y-x) - \frac{1}{6} \nabla^3 f(x)(y-x, y-x, y-x)| \le \frac{L}{24}\|y-x\|^4.
$$
\end{lemma}

\begin{proof}
The proof follows from integration from $x$ to $y$ repeatedly.

First we have
$$
\nabla^2 f(x+u(y-x)) = \nabla^2 f(x) + \left[\int_0^u \nabla^3 f(x+v(y-x)) dv\right](y-x).
$$

By the Lipschitz condition on third order derivative, we know
$$
\|\nabla^3 f(x+v(y-x)) - \nabla^3 f(x)\|_F \le Lv\|x-y\|.
$$
Combining the two we have
$$
\nabla^2 f(x+u(y-x)) = \nabla^2 f(x) + [\nabla^3 f(x)](y-x) + h(u),
$$
where $h(u) = \left[\int_0^u (\nabla^3 f(x+v(y-x)) - \nabla^3 f(x)) dv\right](y - x)$, so $\|h(u)\|_F \le \frac{L}{2}\|x-y\|^2$.

Now we use the integral for the gradient of $f$:
\begin{align*}
\nabla f(x+t(y-x)) & = \nabla f(x) + \left[\int_0^t \nabla^2 f(x+u(y-x))du\right](y-x) \\ &= \nabla f(x) + \nabla^2 f(x) (y-x) + \left[\int_0^t h(u) du\right](y-x).
\end{align*}

Let $g(t) = \left[\int_0^t h(u) du\right](y-x)$, by the bound on $h(u)$ we know $\|g(t)\| \le \frac{1}{6} \|x-y\|^3$. Finally, we have
\begin{align*}
f(y) & = f(x) + \inner{\left[\int_0^1 \nabla f(x+t(y-x))du\right],(y-x)} \\ &= f(x) + \inner{\nabla f(x), y-x} + \frac{1}{2}(y-x)^\top \nabla^2 f(x) (y-x) + \frac{1}{6} \nabla^3 f(x)(y-x)^{\otimes 3} + \inner{\left[\int_0^1 g(t) dt\right],y-x}.
\end{align*}

The last term is bounded by $\|y-x\| \int_0^1\|g(t)\| dt \le \frac{L}{24}\|x-y\|^4$.
\end{proof}

\begin{theorem}[Theorem~\ref{thm:thirdcondition} restated]
Given a function $f$ that satisfies Assumption~\ref{assump:lipthird}, a point $x$ is third order optimal if and only if it satisfies Condition~\ref{cond:third}.
\end{theorem}

\begin{proof}
(necessary condition $\to$ third order minimal) By Lemma~\ref{lem:lipbound} we know
$$f(y) \ge f(x) + \inner{\nabla f(x), y-x} + \frac{1}{2}(y-x)^\top \nabla^2 f(x) (y-x) + \frac{1}{6} \nabla^3 f(x)(y-x)^{\otimes 3}- \frac{L}{24}\|y-x\|^4.$$

Now let $\alpha$ be the smallest nonzero eigenvalue of $\nabla^2 f(x)$. Let $U$ be nullspace of $\nabla^2 f(x)$ and $V$ be the orthogonal subspace. We break $\nabla^3 f(x)$ into two tensors $G_1$ and $G_2$, where $G_1$ is the projection to $V\otimes V\otimes V$, $V\otimes V\otimes U$ (and its symmetries), and $G_2$ is the projection to $V\otimes U\otimes U$ (and its symmetries). Note that $\nabla^3 f(x) = G_1+G_2$ because the projection on $U\otimes U\otimes U$ is 0 by the third condition. Let $\beta$ be the max injective norm of $G_1$ and $G_2$.

Now we know for any $u\in U$ and $v\in V$,
$$
f(x+u+v) - f(x) \ge \frac{1}{2} \alpha \|v\|^2 - \frac{\beta}{6} \|u\| \|v\|^2 - \frac{\beta}{6} \|u\|^2 \|v\| - \frac{L}{24}\|u+v\|^4.
$$

Now, if $\epsilon < \beta/\alpha$, because $\|u\|_2 \le \epsilon$ it is easy to see the sum of first two terms is at least $\frac{1}{3} \alpha \|v\|_2^2$. Now we can take the mininum of
$$
\frac{\alpha}{3} \|v\|^2 - \frac{\beta}{6}\|u\|^2 \|v\|,
$$

The minimum is achieved when $\|v\| = \|u\|^2\beta/\alpha$ and the minimum value is $- \|u\|^4\beta^2/6\alpha$. Therefore when $\|u+v\|\le \beta/\alpha$ we have

$$
f(x+u+v) - f(x) \ge - \left(\frac{\beta^2}{\alpha}+\frac{L}{24}\right) \|u+v\|^4.
$$

(third order minimal$\to$necessary condition)  Assume towards contradiction that the necessary condition is not satisfied, but the point $x$ is third order local optimal.

If the necessary condition is not satisfied, then one of the three cases happens:

In the first case the gradient $\nabla f(x) \ne 0$. In this case, if we let $L'$ be an upperbound the operator norms of the second and third order derivative, then we know
$$
f(x+\epsilon \nabla f(x)) \le f(x) - \epsilon \|\nabla f(x)\|^2 + \frac{\epsilon^2L'}{2}\|\nabla f(x)\|^2 +\frac{\epsilon^3 L'}{6}\|\nabla f(x)\|^3 + \frac{\epsilon^4L}{24}\|\nabla f(x)\|^4.
$$

When $\epsilon\|\nabla f(x)\| \le 1$ and $\epsilon (2L'/3+L/24) \le 1/2$, we have
$$
f(x+\epsilon \nabla f(x)) \le f(x) - \frac{\epsilon}{2} \|\nabla f(x)\|^2. 
$$

Therefore the point cannot be a third order local minimum.

In the second case, $\nabla f(x) = 0$, but $\lambda_{min} \nabla^2 f(x) < 0$. Let $\|u\|=1$ be a unit vector such that $u^\top (\nabla^2 f(x))u = -c < 0$. Let $L'$ be the injective norm of $\nabla^3 f(x)$, then

$$
f(x + \epsilon u) \le f(x) - \frac{c\epsilon^2}{2} + \frac{\epsilon^3 L'}{6} + \frac{\epsilon^4 L}{24}.
$$

Therefore whenever $\epsilon < \min\{\sqrt{3c/L},3c/4L'\}$ we have $f(x+\epsilon u) \le f(x) - \frac{c\epsilon^2}{4}$. The point $x$ cannot be a third order local minimum.

The third case is if $\nabla f(x) = 0$, $\nabla^2 f(x)$ is positive semidefinite, but there is a direction $\|u\|=1$ such that $u^\top (\nabla^2 f(x)) u = 0$, but $[\nabla^3 f(x)](u,u,u) \ne 0$. Without loss of generality we assume $[\nabla^3 f(x)](u,u,u) = c > 0$ (if it is negative we take $-u$), then
$$
f(x+\epsilon u) \le f(x) - c\epsilon^3/6 + L\epsilon^4/24.
$$ 
Therefore whenever $\epsilon < 2c/L$ we have $f(x+\epsilon u) \le f(x) - c\epsilon^3/12$ so $x$ cannot be a third order optimal.
\end{proof}

\subsection{Algorithm for Competitive Subspace, Proof of Theorem~\ref{thm:approx}}

\begin{algorithm}
\begin{algorithmic}
\REQUIRE Function $f$, point $z$, Hessian $M = \nabla^2 f(z)$, third order derivative $T = \nabla^3 f(z)$, approximation ratio $Q$, Lipschitz Bound $L$, 
\ENSURE Competitive subpace $\Sc(z)$ and $C_Q(z)$.
\STATE Compute the eigendecomposition $M = \sum_{i=1}^n \lambda_i v_iv_i^\top$.
\FOR{$i = 1$ \TO $n$}
\STATE Let $\Sc = \mbox{span}\{v_i,v_{i+1},...,v_n\}$.
\STATE Let $C_Q = \|\mbox{Proj}_\Sc T\|_F$.
\IF{$\frac{C_Q^2}{12LQ^2} \ge \lambda_i$}
\RETURN $\Sc, C_Q$.
\ENDIF
\ENDFOR
\RETURN $\Sc = \emptyset, C_Q = 0$.
\end{algorithmic}
\caption{Algorithm for computing the competitive subspace\label{alg:compet}}
\end{algorithm}

\begin{theorem}[Theorem~\ref{thm:approx} restated]
There is a universal constant $B$ such that the expected number of iterations of  Algorithm~\ref{alg:approx} is at most $2$, and the output of $\mbox{Approx}$ is a unit vector $u$ that satisfies $T(u,u,u) \ge \|\mbox{Proj}_\Sc T\|_F/Q$ for $Q = Bn^{1.5}$.
\end{theorem}

\begin{proof}
We use the anti-concentration property for Gaussian random variables 
\begin{theorem}[anti-concentration\citep{carbery2001distributional}]
Let $x\in \Rbb^n$ be a Gaussian variable $x\sim N(0,I)$, for any polynomial $p(x)$ of degree $d$, there exists a constant $\kappa$ such that
$$
\Pr[|p(x)|\le \epsilon \sqrt{\Var[p(x)]}] \le \kappa \epsilon^{1/d}.
$$
\end{theorem}

In our case $d = 3$ and we can choose some universal constant $\epsilon$ such that the probability of $p(x)$ being small is bounded by $1/3$. It is easy to check that the variance is lowerbounded by the Frobenius norm squared, so $$
\Pr[|T(\hat{u},\hat{u},\hat{u})|\ge \epsilon \|\mbox{Proj}_\Sc T\|_F] \ge 2/3.
$$

On the other hand with high probability we know the norm of the Gaussian $\hat{u}$ is at most $2\sqrt{n}$. Therefore with probability at least $1/2$, $|T(\hat{u},\hat{u},\hat{u})| \ge \epsilon \|\mbox{Proj}_\Sc T\|_F$ and $\|\hat{u}\| \le 2\sqrt{n}$, therefore $|T(u,u,u)| \ge \frac{\epsilon}{8n^{1.5}} \|\mbox{Proj}_\Sc T\|_F$. Choosing $B = 8/\epsilon$ implies the theorem. 
\end{proof}
\label{app:approx}
\subsection{Proof of Theorem~\ref{thm:hard}}
\label{app:hard}
\begin{theorem}
[Theorem~\ref{thm:hard} restated]
It is NP-hard to find a fourth order local minimum of a function $f(x)$, even if $f$ is guaranteed to be well-behaved.
\end{theorem}

\begin{proof}
We reduce the problem of verifying nonnegativenss of degree 4 polynomial to the problem of finding fourth order local minimum.

Given a degree 4 homogeneous polynomial $f(x)$, we can write it as a symmetric fourth order tensor $T\in \Rbb^{n^4}$. Without loss of generality we can rescale $T$ so that $\|T\|_F \le 1$ and therefore $\|T\| \le 1$. 

Now we define the function $g(x) = f(x)+\|x\|^6$. We first show that this function is well-behaved.

\begin{claim}
$g(x)$ is well-behaved.
\end{claim}

\begin{proof}
 Since $g(x)$ is a polynomial with bounded coefficients, clearly it is infinite order differentiable and satisfies condition 2. For condition 1, notice that $g(x) = 0$ and for all $\|x\|_2 > 1$, we have $g(x) \ge \|x\|^6-\|x|^4 > 0$ so the global minimizer must be at a point within the unit $\ell_2$ ball. Finally, for any $\|x\| = 1$, we know $g(tx) = f(x)t^4 + t^6$ which is always increasing when $t\ge 1$ since $|f(x)| \le 1$.
 \end{proof}

Next we show if $f(x)$ is nonnegative, then $\vec{0}$ is the unique fourth order local minimizer.

\begin{claim}
If $f(x)$ is nonnegative, then $\vec{0}$ is the unique fourth order local minimizer of $g(x)$.
\end{claim}

\begin{proof}
Suppose $x\ne 0$ is a local minimizer of $g(x)$ of order at least 1. Let $u = x/\|x\|$. We consider the function $g(tu) = f(u)t^4 + t^6$. Clearly the only first order local minimizer of $g(tu)$ is at $t = 0$. Therefore $x$ cannot be a first order local minimizer of $g(x)$. 
\end{proof}

Finally, we show if $f(x)$ has a negative direction, then all the local minimizer of $g(x)$ must have negative value in $f$.

\begin{claim}
If $f(x)$ is negative for some $x$, then if $x$ is a fourth order local minimum of $g(x)$ then $f(x) < 0$.
\end{claim}

\begin{proof}
Suppose $x\ne 0$ is a fourth order local minimum of $g(x)$. Then at least $t=1$ should be a fourth order local minimum of $g(tx) = f(x)t^4 +t^6\|x\|^6$. This is only possible if $f(x) < 0$. 

On the other hand, for $x = 0$, suppose $\|z\|=1$ is a direction where $f(z) < 0$, then $f(x) - f(x+tz) = f(z)t^4 - t^6 = \Omega(t^4)$, so $x = 0$ is not a fourth order local minimum.
\end{proof}

The theorem follows immediately from the three claims.
\end{proof}

\end{document}